\documentclass[twocolumn]{article}%

\usepackage{fullpage}
\usepackage{graphicx}
\usepackage{mathptmx}
\usepackage[utf8]{inputenc} 
\usepackage[T1]{fontenc}    
\usepackage{hyperref}       
\usepackage{url}            
\usepackage{booktabs}       
\usepackage{amsmath,amsthm,amssymb} 
\usepackage{amsfonts}       
\usepackage{nicefrac}       
\usepackage{microtype}      
\usepackage{graphicx}
\usepackage{multirow}
\usepackage{multicol}
\usepackage{amsmath}
\usepackage{mathtools}
\usepackage{lipsum}
\usepackage{xcolor}
\usepackage{authblk}

\usepackage{pdflscape}
\usepackage{afterpage}
\usepackage{rotating}

\definecolor{dkcbred}{HTML}{AA4B00}
\definecolor{dkcbblue}{HTML}{005C8E}
\definecolor{dkcbgreen}{HTML}{017E5C}
\definecolor{dkcbyellow}{HTML}{B17204}

\newcommand{\rlipsum}[1][1-7]{{\color{red}\lipsum[#1]}}

\DeclarePairedDelimiterX{\norm}[1]{\lVert}{\rVert}{#1}
\DeclarePairedDelimiterX{\vecnorm}[1]{|}{|}{#1}

\graphicspath{{Fig/}}

\newtheorem{theorem}{Theorem}
\newbox{\myorcidaffilbox}
\sbox{\myorcidaffilbox}{\large\includegraphics[height=1.7ex]{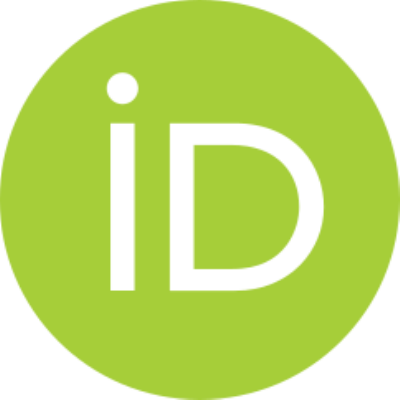}}
\newcommand{\ORCID}[1]{\href{https://orcid.org/#1}{\usebox{\myorcidaffilbox}}}

\begin{document}

\title{Subset-Contrastive Multi-Omics Network Embedding}

\author[1,2]{Pedro Henrique da Costa Avelar\ORCID{0000-0002-0347-7002}}
\author[2]{Min Wu\ORCID{0000-0003-0977-3600}}
\author[1]{Sophia Tsoka\ORCID{0000-0001-8403-1282}}

\affil[1]{Department of Informatics, Faculty of Natural, Mathematical and Engineering Sciences, King's College London, Bush House, WC2B 4BG, London, United Kingdom}
\affil[2]{Institute for Infocomm Research, Agency for Science Technology and Research (A*STAR), 1 Fusionopolis Way, \#21-01 Connexis, 138632, Singapore, Singapore}

\twocolumn[
  \begin{@twocolumnfalse}
\maketitle
\begin{abstract}
\textbf{Motivation:} Network-based analyses of omics data are widely used, and while many of these methods have been adapted to single-cell scenarios, they often remain memory- and space-intensive. As a result, they are better suited to batch data or smaller datasets. Furthermore, the application of network-based methods in multi-omics often relies on similarity-based networks, which lack structurally-discrete topologies. This limitation may reduce the effectiveness of graph-based methods that were  initially designed for  topologies with better defined structures\\
\textbf{Results:} We propose Subset-Contrastive multi-Omics Network Embedding (SCONE), a method that employs contrastive learning techniques on large datasets through a scalable subgraph contrastive approach. By exploiting the pairwise similarity basis of many network-based omics methods, we transformed this characteristic into a strength, developing an approach that aims to achieve scalable and effective analysis. Our method demonstrates synergistic omics integration for cell type clustering in single-cell data. Additionally, we evaluate its performance in a bulk multi-omics integration scenario, where SCONE performs comparable to the state-of-the-art despite utilising limited views of the original data. We anticipate that our findings will motivate further research into the use of subset contrastive methods for omics data. \\
\textbf{Availability:} Source code and environment definitions for the experiments will be made available after publication at \url{github.com/phcavelar}.
\\ \textbf{Contact:} \href{pedro_henrique.da_costa_avelar@kcl.ac.uk}{pedro\_henrique.da\_costa\_avelar@kcl.ac.uk}, \href{email:wumin@i2r.a-star.edu.sg}{wumin@i2r.a-star.edu.sg}, \href{sophia.tsoka@kcl.ac.uk}{sophia.tsoka@kcl.ac.uk}
\end{abstract}
    \end{@twocolumnfalse}
]

\section{Introduction}

Precision medicine has emerged as a focal point of contemporary biomedical research, fuelled largely by advancements in high-throughput sequencing technologies that enable the characterization of samples at molecular level, often across multiple data types \cite{nicora_integrated_2020}. This growing interest underlies the need for holistic approaches to understand complex diseases such as cancer \cite{bode_precision_2017}, addressing rapid advancements in single-cell and spatial multi-omics profiling technologies \cite{baysoy_technological_2023}. Leveraging these developments, recent efforts have introduced technologies like scGET-seq, which integrate four omics layers, and Spatial ATAC–RNA-seq, which achieves near-cell resolution multi-omics \cite{zhang_spatial_2023}.

While the increasing availability of multi-omics data has enabled novel discoveries, network-based analyses have long been recognized as powerful tools for uncovering the relationships between topological and functional features of biological systems \cite{xu_module_2010,bennett_detection_2015,amiri_souri_novel_2022,fernandez-torras_integrating_2022}. 
Such methods have been applied not only to biological systems at large but also to the analysis of single- and multi-omics data across batch, single-cell, and spatial omics contexts.

Recent advances in machine learning, particularly in deep learning, have significantly expanded the potential of network-based methods. The development of graph embedding techniques \cite{goyal_graph_2018,cui_survey_2019} and knowledge graph embeddings \cite{wang_knowledge_2017,dai_survey_2020,yan_survey_2022} has played a key role in this progress, with many of these methods becoming integral to the standard toolkit for biological analyses \cite{amiri_souri_novel_2022,fernandez-torras_integrating_2022}. Building on these foundations, Graph Contrastive Learning (GCL) methods \cite{ju_survey_2024,wang_contrastive_2024} have emerged as a powerful approach for leveraging the structure of network-based data. Despite these advancements, the application of network analysis in omics data often relies on K-nearest neighbors (KNN) graphs or correlation/interaction graphs, which lack the discrete topological structures that many graph-based methods were originally designed to handle.

Although GCL methods have been applied to multi-omics data, they often fail to account for the limitations of non-concrete graph structures, such as those based on correlation or similarity measures. These types of graphs lack the discrete topological connections that many graph-based methods are designed to leverage. Additionally, in the context of single-cell omics, where sample sizes are large, existing GCL methods that require comparisons between all node pairs (e.g., SERIES from \cite{wang_contrastive_2024}, Subtype-MGTP from \cite{xie_subtype-mgtp_2024} or SC-ZAG from \cite{zhang_sczag_2024}) become computationally infeasible due to their quadratic memory requirements.

In this study, we propose a Subset-Contrastive learning approach tailored to network-based multi-omics data. We introduce SCONE, a method designed to address the challenges posed by approximate graph structures, such as in KNN graphs, by leveraging scalable contrastive learning techniques. Through comparative analysis with two strong baselines, we demonstrate that SCONE achieves competitive performance in both bulk and single-cell multi-omics scenarios. Additionally, we show that, with appropriate hyperparameter selection, SCONE offers reduced memory complexity compared to full-graph methods. While this study focuses on a specific configuration of our model, we hope that our findings will highlight the potential of subset contrastive learning for omics data and inspire further research in this area.

\section{Related Work}

Multi-omics representation learning, particularly for bulk data integration, is a key approach for uncovering complex biological relationships to enhance downstream analyses and predictive modelling. One notable contribution is MOFA \cite{argelaguet_multiomics_2018,argelaguet_mofa_2020}, a widely used statistical framework originally developed for single-cell data but adopted as a baseline for bulk integration. For instance, \cite{lee_variational_2021} demonstrated the feasibility of employing a Product-of-Experts (PoE) model to integrate bulk multi-omics datasets from the TCGA network, achieving improved performance compared to MOFA.

In parallel, generative and contrastive learning methods have attracted attention for multi-view representation learning. In this paper the terms multi-view, multi-modal, and multi-omics are used interchangeably. Subtype-GAN \cite{yang_subtype-gan_2021} introduced a generative adversarial network (GAN) within a multi-modal autoencoder framework, enabling joint representation learning of multi-omics datasets. Subsequent advancements include contrastive learning approaches such as Subtype-DCC \cite{zhao_subtype-dcc_2023}, DMCL \cite{chen_deep_multi-view_2023}, and MOCSS \cite{chen_mocss_2023}, which focus on improving multi-omics integration through contrastive objectives tailored to multi-view datasets. These methods highlight the potential of contrastive and generative approaches to address challenges in multi-omics data integration.

Recent advancements have also emphasized the integration of Graph Contrastive Learning (GCL) techniques into bulk multi-omics data analysis. A prominent example is SubtypeMGTP \cite{xie_subtype-mgtp_2024}, which leverages a two-step algorithm comprising multi-omics-to-protein translation and network-based deep subspace clustering. This method demonstrated significant performance improvements over previous approaches.

With the rapid evolution of single-cell and spatial multi-omics technologies \cite{baysoy_technological_2023}, there has been an increased focus on the development of methods tailored for single-cell datasets. For example, \cite{lotfollahi_multigrate_2021} extended the PoE framework proposed by \cite{lee_variational_2021} to single-cell scenarios, providing robust integration of omics layers. Similarly, TotalVI \cite{gayoso_joint_2021}—a widely recognized multi-modal Variational Autoencoder—has been applied to single-cell RNA and protein data, fostering its integration into mainstream processing pipelines \cite{gayoso_python_2022,virshup_scverse_2023,ergen_scvi-hub_2024}. In the realm of graph-based methodologies, GraphST \cite{long_spatially_2023} employed KNN graphs derived from spatial RNA and transcriptomic data, integrating Deep Graph Infomax (DGI) \cite{velickovic_deep_2019} to enhance spatial clustering. Another method, scZAG \cite{zhang_sczag_2024}, incorporated a Zero-Inflated Negative Binomial Variational Autoencoder \cite{lopez_deep_2018} alongside GCL to advance single-cell clustering.

Despite these advancements, the application of Graph Contrastive Learning to omics data often relies on graph structures derived from correlation or similarity measures, which lack the discrete topologies that many GCL methods were specifically designed to exploit. Existing GCL-based models for single-cell data, such as SERIES \cite{wang_contrastive_2024}, Subtype-MGTP \cite{xie_subtype-mgtp_2024}, and scZAG \cite{zhang_sczag_2024}, require quadratic memory resources, rendering them computationally expensive for large datasets.

In parallel, Subset-Contrastive Learning has emerged as a promising avenue to address scalability challenges in graph-based learning. The SubG-Con method proposed by \cite{jiao_sub-graph_2020} was specifically designed to address the scalability issues inherent in earlier methods such as Deep Graph Infomax (DGI) \cite{velickovic_deep_2019} and Graphical Mutual Information (GMI) \cite{peng_graph_2020}, which often require quadratic memory resources and struggle with large datasets. SubG-Con introduced an efficient approach by partitioning graphs into subgraphs. This was further refined by methods like AdaGCL \cite{wang_adagcl_2022}, which introduced batch-aware view generation, and \cite{xu_graph_2023}, which optimized subset selection using low dissimilarity priority and mutual exclusion principles. However, to the best of our knowledge, these subset-contrastive approaches have not yet been adapted to omics data.

Our work aims to bridge this gap by introducing Subset-Contrastive multi-Omics Network Embedding (SCONE), which leverages subgraph-based contrastive learning to address scalability and computational challenges in multi-omics data integration. Unlike existing methods, SCONE applies subset contrastive techniques to omics datasets, demonstrating robust performance across both single- and multi-omics scenarios while maintaining computational efficiency.

\section{Materials}

The proposed model was applied on two datasets. First, a dataset containing Peripheral Blood Mononuclear Cells (PBMC) from \cite{kotliarov_broad_2020}, with single-cell RNA sequencing (RNA-seq) and surface protein (CITE-seq) measurements was used. Annotations have been derived from CITE-seq clustering and subsequently validated through RNA-seq expression. Following \cite{lotfollahi_multigrate_2021}, quality control was performed, with RNA-seq data normalized to sum to 10,000 and transformed using $\log{(x+1)}$. Protein counts were normalized using the centered log-ratio transformation \cite{stoeckius_simultaneous_2017}. Data were z-scored to mean of 0 and a standard deviation of 1 before being used as input to the neural networks. After preprocessing, the dataset contained 52,117 cells from 20 samples of high vs. low responders. Cell types were clustered and annotated based on surface protein information, resulting in 10 broad cell types at the first level of annotation, as detailed in Table~\ref{tab:dset-kotliarov-cell-types}.

\begin{table}[hptb]
    \centering
    \begin{tabular}{lrr}
    \toprule
    Cell Type & Count & Proportion \\
    \midrule
    Total & 52,117 & $100.00\%$\\
    \midrule
    CD4 naive & 11,127 & $21.35\%$\\
    CD4 memory T & 9,161 & $17.58\%$\\
    Classical monocytes and mDC & 6,631 & $12.73\%$ \\
    B & 5,810 & $11.15\%$ \\
    CD8 memory T & 5,268 & $10.11\%$ \\
    CD8 naive & 4,840 & $9.29\%$ \\
    NK & 4,700 & $9.02\%$ \\
    Unconventional T cells & 3,268 & $6.27\%$ \\
    Non-classical monocytes & 1,072 & $2.06\%$ \\
    pDC & 240 & $0.46\%$ \\
    \bottomrule
    \end{tabular}
    \caption{Cell Type composition of the dataset from \cite{kotliarov_broad_2020}}
    \label{tab:dset-kotliarov-cell-types}
\end{table}

Second, a dataset derived from The Cancer Genome Atlas (TCGA) was used as preprocessed previously \cite{yang_subtype-gan_2021} and \cite{xie_subtype-mgtp_2024}. It comprises five omics layers across seven cancer types: RNA (\texttt{rna}), miRNA (\texttt{miRNA}), methylation (\texttt{meth}), copy number (\texttt{CN}), and protein (\texttt{prot}). The RNA, miRNA, methylation, and copy number layers were preprocessed by \cite{yang_subtype-gan_2021}, while the protein layer was provided by \cite{xie_subtype-mgtp_2024}. The dataset encompasses seven cancer types: Bladder Cancer (BLCA), Breast Cancer (BRCA), Kidney Clear Cell Cancer (KIRC), Lung Adenocarcinoma (LUAD), Skin Cutaneous Melanoma (SKCM), Stomach Adenocarcinoma (STAD), and Uterine Corpus Endometrial Carcinoma (UCEC). In total, these sub-datasets include 3,013 samples, with detailed information presented in Table~\ref{tab:dset-tcga}.

\begin{table}[hptb]
    \footnotesize
    \centering
    \begin{tabular}{crrrr}
    \toprule
        Cancer & Samples & Male/Female \% & Age & Clusters \\
    \midrule
        BLCA & 331 & 74.92~/~25.08 & $67.88\pm10.63$ & 5 \\
        BRCA & 829 & *1.09~/~98.79 & $58.22\pm13.39$ & 5 \\
        KIRC & 438 & 65.98~/~34.02 & $60.38\pm12.07$ & 4 \\
        LUAD & 346 & 46.53~/~53.47 & $64.57\pm~9.86$ & 3 \\
        SKCM & 333 & 59.46~/~40.54 & $58.13\pm15.79$ & 4 \\
        STAD & 329 & 66.57~/~33.43 & $65.41\pm11.01$ & 3 \\
        UCEC & 407 &  0.00~/~100.0 & $63.64\pm11.26$ & 4 \\
    \bottomrule
    \end{tabular}
    \caption{The composition of the TCGA dataset used for the bulk analysis. Age is given as $\text{mean}\pm\text{standard deviation}$. The number of clusters is derived from the clustering results reported by \cite{yang_subtype-gan_2021} and \cite{xie_subtype-mgtp_2024}. * Gender totals do not sum to 100\% due to missing values.}
    \label{tab:dset-tcga}
\end{table}

Although \cite{yang_subtype-gan_2021} and \cite{xie_subtype-mgtp_2024} also included data for Uveal Melanoma (UVM), this dataset was excluded due to its limited sample size (n=12 for the \texttt{CN} and protein layers), which was deemed insufficient for the purposes of this study. Furthermore, since the strongest baseline model performed poorly on this dataset—likely due to its small size relative to the other datasets—its exclusion is not expected to impact the overall conclusions of this work. Unlike \cite{xie_subtype-mgtp_2024}, each dataset was treated independently, with no information shared across datasets. To evaluate the subtyping capabilities of the proposed model, only samples containing all five omics layers were selected. Additionally, during downstream processing, samples missing information were excluded.

\section{Methodology}

\begin{figure*}
    \centering
    \includegraphics[width=.95\linewidth]{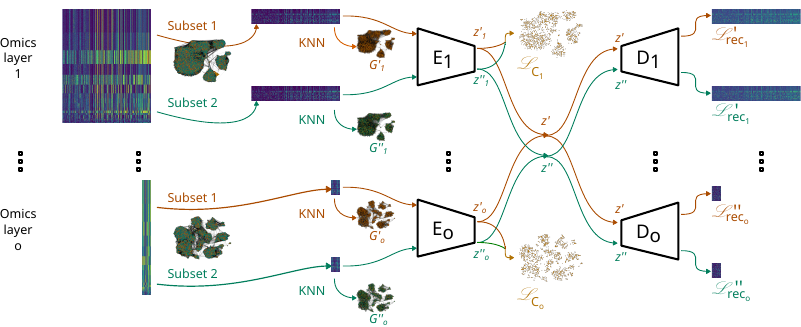}
    \caption{A schematic representation of the proposed SCONE model. The model begins by subsetting the original dataset, sampling $k_{s}$ samples for each subset in each omics layer, resulting in {\color{dkcbred}subset 1} and {\color{dkcbgreen}subset 2}. For each subset, a K-nearest neighbors graph is constructed for each omics layer. The omics measurements and graph of each subset are then input into an omics-specific GAT encoder $E_{i}$ to produce the latent representation $z_{i} = E_{i}(x_{i}, G_{i})$. Subset-specific representations are denoted as $z^{\prime}{i}$ for {\color{dkcbred}subset 1} and $z^{\prime\prime}{i}$ for {\color{dkcbgreen}subset 2}. The outputs from the encoder layers are combined into a shared representation $z$, which is subsequently passed through the omics-specific GAT decoder $D_{i}$ to reconstruct each omics layer for each subset. We optimise the model through the reconstruction loss $\mathcal{L}_{\text{rec}_{i}}$. Simultaneously, the {\color{dkcbyellow}overlapping nodes} between subsets are leveraged for the contrastive loss $\mathcal{L}_{C_{i}}$, which utilizes the encoder outputs and neighborhood information for both overlapping and non-overlapping nodes.
    Components with a {\color{dkcbred}orange colour} show the trajectory under {\color{dkcbred}subset 1}, equally the {\color{dkcbgreen}green colour} represents {\color{dkcbgreen}subset 2}. The components between the encoders and decoders in {\color{dkcbyellow}yellow} show the {\color{dkcbyellow}overlap} between both sets, which is used for our {\color{dkcbyellow}subset contrastive loss}. The heatmaps shown correspond to the \cite{kotliarov_broad_2020} dataset, with omics layer 1 being the RNA data and Omics layer o being CITE-seq data. For readability we reduce the number of samples, RNA probes, and CITE-seq probes by a factor of roughly 101, 10, and 2, respectively, with subsets further reducing the number of samples by approximately 10. }
    \label{fig:scone}
\end{figure*}

This section outlines the methodology employed in our experiments. We first provide a detailed description of the Graph Attention Network (GAT) layer. Subsequently, we present the framework of our proposed model. Finally, we describe the experimental pipelines and the performance metrics utilized. Additional experimental details can be found in \nameref{sup}'s \nameref{sup:exp} section.

\subsection{Graph Attention Networks (GATs)}

Graph Attention Networks (GATs) are a Graph Neural Network architecture introduced by \cite{velickovic_graph_2018}. They consist of multiple graph attention layers, each of which processes an input matrix $h^{l-1} \in R^{n \times d_{l-1}}$ and a graph $G = (V,E)$ to produce an output matrix $h^{l} \in R^{n \times d_{l}}$. Attention factors $\alpha_{i,j}$ are computed between a node $i \in V$ and its neighbours $j \in \mathcal{N}(i)$, where $\mathcal{N}(i)$ is the neighbourhood of $i$ defined by the edge list $E$. Specifically, $\exists \alpha_{i,j} \iff \exists (j,i) \in E$. These factors are typically computed using Equation~\ref{eq:attention}, where $e:\mathbb{R}^{d_{l-1}} \times \mathbb{R}^{d_{l-1}} \rightarrow \mathbb{R}^{1}$ is a function that quantifies the similarity between two input vectors.

\begin{equation} \label{eq:attention}
    \alpha_{i,j} = \frac{\exp{(e(h^{l-1}_{i},h^{l-1}_{j})}}{\sum_{k \in \mathcal{N}(i)}{\exp{(e(h^{l-1}_{i},h^{l-1}_{k})}}}.
\end{equation}

These attention factors weight the neighborhood aggregation function, enabling the computation of a node's updated representation as a weighted sum of its neighbors' features, transformed by a nonlinear layer with weight $W_{l}$ and nonlinearity $\sigma$, as shown in Equation~\ref{eq:gat}:

\begin{equation} \label{eq:gat}
\begin{aligned}
    h^{l} &= \operatorname{GAT}(h^{l-1},G) \\
    \text{where~} h^{l}_{i} &= \sigma\left(\sum_{j \in \mathcal{N}(i)} \alpha_{i,j} h^{l-1,\prime}_{j}\right) \\
    \text{and~} h^{l-1,\prime} &= h^{l-1} \times W_{l}.
\end{aligned}
\end{equation}

In the original GAT implementation \cite{velickovic_graph_2018}, the attention factor is defined as in Equation~\ref{eq:attention-velickovic}, where $W \in \mathbb{R}^{d_{h}\times d_{h}^{\prime}}$, and $a \in \mathbb{R}^{2d^{\prime} \times 1}$ are the learnable parameters, and $\|$ denotes concatenation along the feature dimension:

\begin{equation} \label{eq:attention-velickovic}
    e(h_{i},h_{j}) = \operatorname{LeakyReLU}([(h_{i}\times W)\|(h_{j}\times W)] \times a).
\end{equation}

However, it has been demonstrated that the attention mechanism originally developed for GAT employs a \emph{static attention} mechanism. This design limits its adaptability, as it cannot attend to different keys (i.e., vectors indexed by $j$, representing the neighbors of a node) when provided with varying queries (i.e., vectors indexed by $i$, representing the targets of the edges). To address this limitation, we employ the $\operatorname{GAT}$ operator implemented through the \texttt{GATv2Conv} module \cite{brody_how_2022} in the \texttt{pytorch-geometric} library \cite{fey_fast_2019} (version \texttt{2.5.2}). This operator introduces a \emph{dynamic attention} mechanism, defined in Equation~\ref{eq:attention-brody}, where a learned linear transformation $W \in \mathbb{R}^{2d \times d^{\prime}}$ is applied after concatenation, followed by a nonlinearity and another linear layer $a \in \mathbb{R}^{d \times 1}$. This approach effectively employs a multilayer perceptron (MLP), which is a universal function approximator \cite{hornik_multilayer_1989}, to compute the attention scores:

\begin{equation} \label{eq:attention-brody}
    e(h_{i},h_{j}) = \operatorname{LeakyReLU}([(h_{i}\|h_{j}W) \times W]) \times a.
\end{equation}

\subsection{Subset-Contrastive multi-Omics Network Embedding (SCONE) framework}

The Subset-Contrastive multi-Omics Network Embedding (SCONE) framework was developed by adapting a Gaussian-prior Product of Experts (PoE) \cite{hinton_training_2002} Variational Autoencoder (VAE) \cite{lee_variational_2021} into a Multi-modal Deep Autoencoder \cite{ngiam_multimodal_2011}. Unlike \cite{ngiam_multimodal_2011}, SCONE does not rely on pre-training using Restricted Boltzmann Machines (RBMs). Instead, all pre-training is performed through contrastive learning inspired by Deep Graph Infomax (DGI) \cite{velickovic_deep_2019}. The encoder and decoder architectures employ Graph-based modules, specifically using Graph Attention Networks (GATs) \cite{velickovic_graph_2018,brody_how_2022}. A detailed description of the framework is provided below and visualized in Figure~\ref{fig:scone}.

The SCONE framework handles $\left|O\right|$ Omics views, where each omics view $o \in O$ contains a data table $x_{o} \in \mathbb{R}^{n_{o} \times d_{o}}$ with $n_{o}$ samples $i \in S_{o}$ (e.g. single cells or bulk samples) and $d_{o}$ features (e.g., gene or protein expression probes).  The framework’s goal is to construct individual latent spaces $z_{o} \in \mathbb{R}^{n_{o} \times d_{z}}$ for each omics view, which are then combined into a common latent space $z  \in \mathbb{R}^{n \times d_{z}}$. From $z$, the model reconstructs each input table $\hat{x}_{o}$ and then applies a reconstruction loss $\mathcal{L}_{\text{rec}_{o}}(x_{o}, \hat{x}_{o})$ for each omics layer.

To achieve this, omics-specific encoders $E_{o}$ are constructed using GATs. As Graph Attention Networks require a graph as input and omics data do not inherently form a network, a KNN-graph $G_{o} = (V_{o},E_{o})$ is generated for each omic layer, connecting each sample $i \in V_{o}$ to its $k_{k}$ nearest neighbours. For all experiments, $k_{k}=15$ is used. Each omics measurement produces a hidden latent space as described in Equation~\ref{eq:encoder-full}.

\begin{equation}\label{eq:encoder-full}
\begin{aligned}
     z_{o} &= E_{o}(x_{o}, G_{o})~\forall o \in O\\
    \text{where~} G_{o} &= \operatorname{KNN}(x_{o},k_{k}) \\
    \text{and~} E_{o} &= (\operatorname{GAT}_{E,o,1} \cdots \operatorname{GAT}_{E,o,k})(x_{o}, G_{o}).
\end{aligned}
\end{equation}

These individual latent spaces $z_{o}$ are combined into a joint latent space $z$ through a quantity- and order-invariant pooling operation $p$, as shown in Equation~\ref{eq:latent-full}. Considering that each $z_{o}$ is an element in an Abelian group and $p(a,b)$ being the operation, we can then treat missing values by substituting them with the group's identity $\mathbb{I}$. Without loss of generality, we consider $p$ to be the sum between vectors, and use an empty array of zeros $\mathbb{I} \in \mathbb{R}^{d_{z}}, \mathbb{I}[j] = 0~\forall j$. This adaptation using an Abelian group, alongside other properties that we discuss below, allows us to adapt the Product-of-Experts framework into a non-variational model \cite{hinton_training_2002,lee_variational_2021}.

\begin{equation}\label{eq:latent-full}
\begin{aligned}
     z[i] &= p(\{id(z_{o},i)~\forall o \in O\}) \\
    \text{where~} id(z_{o},i) &= \begin{dcases}
        z_{o}[i] & \text{if~} i \in S_{o}, \\
        \mathbb{I} & \text{otherwise}.
    \end{dcases}
\end{aligned}
\end{equation}

From $z$, each omics layer is reconstructed using omics-specific graph decoders $D_{o}$, which also utilise GATs. Reconstruction is described in Equation~\ref{eq:decoder-full}:

\begin{equation}\label{eq:decoder-full}
\begin{aligned}
     \hat{x}_{o} &= D_{o}(z_{o}, G_{o})~\forall o \in O\\
    \text{where~} G_{o} &= \operatorname{KNN}(x_{o},k_{k}) \\
    \text{and~} D_{o} &= (\operatorname{GAT}_{D,o,1} \cdots \operatorname{GAT}_{D,o,k})(x_{o}, G_{o}).
\end{aligned}
\end{equation}

The reconstruction loss $\mathcal{L}_{\text{rec}_{o}}(x_{o}, \hat{x}_{o})$ is applied to each of the reconstructed output where ground-truth is available. This allows the network to reconstruct each omics layer from its available data from a joint latent space which might not contain all information. For normally-distributed input, the mean-squared error (MSE) is used (Equation~\ref{eq:rec-mse}). However, we have also implemented in our software losses for Negative Binomial (NB), Zero-Inflated Negative Binomial (ZINB) and Binary inputs (which can be seen in Equation~\ref{eq:rec-bce}, where we have a probability matrix $p_{o}$ instead of $x_{o}$). These reconstruction functions are commonplace in omics data (e.g. single-cell RNA measurements are commonly represented through an NB/ZINB distribution \cite{lopez_deep_2018,gayoso_joint_2021,lin_clustering_2022}, and methylation beta-values can be interpreted as binary values or probabilities \cite{choi_methcancer-gen_2020}).

\begin{equation} \label{eq:rec-mse}
    \mathcal{L}_{\text{rec-mse}_{o}}(x_{o}, \hat{x}_{o}) = \frac{1}{n_{o}\times d_{o}} \sum_{i=1}^{n_{o}} \sum_{j=1}^{d_{o}} (x_{o_{i,j}} - \hat{x}_{o_{i,j}})^{2}.
\end{equation}

\begin{equation} \label{eq:rec-bce}
\begin{aligned}
    \mathcal{L}_{\text{rec-bce}_{o}}(x_{o}, \hat{x}_{o}) &= \frac{1}{n_{o}\times d_{o}} \sum_{i=1}^{n_{o}} \sum_{j=1}^{d_{o}} (p_{o_{i,j}} \log{\hat{p}_{o_{i,j}}}) + (\overline{p}_{o_{i,j}} \log{\overline{\hat{p}}_{o_{i,j}}}) \\
    \text{where~} & \hat{p}_{o_{i,j}} = \sigma(\hat{x}_{o_{i,j}}) \text{~and~} \overline{\hat{p}}_{o_{i,j}} = 1-\sigma(\hat{x}_{o_{i,j}}).
\end{aligned}
\end{equation}

However, using the full input matrix and induced connectivity Graph may be computationally infeasible, particularly during training, when the number of samples is large, as is often the case in single-cell multi-omics data. To address this, we propose our Subset-Contrastive framework. Instead of using the full omics input matrix $x_{o}$, we subset each omics matrix into two overlapping subsets of samples, $s_1$ and $s_2$, during each training epoch. These subsets are defined such that $\forall i \in s_{j}, i \in \bigcup_{o} S_{o}$, $s_1 \cap s_2 \neq s_1$, $\left|s_1 \cup s_2\right| > k_{s}$. This approach ensures that $s_{1}$ and $s_{2}$ share some samples but are not identical.

Using these subsets, we construct two versions of the omics input matrices: $x^{\prime}_{o} = x_{o}[s_1]$ and $x^{\prime\prime}_{o} = x_{o}[s_2]$. For each subset, we generate corresponding KNN graphs $G^{\prime}_{o}$ and $G^{\prime\prime}_{o}$, resulting in two sets of latent representations of the omics data: $z^{\prime}$ and $z^{\prime}_{o}~\forall o \in O$ for the first subset and $z^{\prime\prime}$ and $z^{\prime\prime}_{o}~\forall o \in O$ for the second subset.

Inspired by DGI \cite{velickovic_deep_2019}, we define a per-omics layer subset-contrastive loss $\mathcal{L}_{\text{C}_{o}}(z^{\prime}_{o}, z^{\prime\prime}_{o})$ as the Binary Cross Entropy of the similarity of the neighbourhood average of positive and negative examples.  By applying this contrastive loss to each individual omics layer representation rather than the joint latent space, the network can learn complementary information about each omics layer independently, which is then aggregated into the joint representation.

Positive examples $s_{+}$ are defined as samples $i$ that appear in both subsets, $i \in (s_1 \cup s_2)$. Negative examples are defined separately for each subset: $s_{-,1}$ includes samples $i \in (s_1 \setminus s_2)$ and $s_{-,2}$ includes samples $j \in (s_2 \setminus s_1)$. to compute probabilities for positive and negative pairs, as shown in Equation~\ref{eq:bilinear}. Here $\dagger$ and $\ddagger$ represent either $\prime$ or $\prime\prime$, depending on the subset being processed.

\begin{equation} \label{eq:bilinear}
\begin{aligned}
    C_{o}(i,\dagger,j,\ddagger) &= n^{\dagger}_{o}[i] \times W_{C,o} \times {n^{\ddagger}_{o}[j]}^{\top}\\
    \text{where~} n_{o}^{\dagger}[i] &= \frac{1}{\left|\mathcal{N}(i)\right|} \sum_{j \in \mathcal{N}(i)} z^{\dagger}_{o}[j].
\end{aligned}
\end{equation}

The contrastive loss is then computed as the BCE loss between the logits of positive and negative pairs, as shown in Equation~\ref{eq:contrastive}. This formulation encourages the model to maximize mutual information between nodes that are shared between subsets (and are truly supposed to be neighbours) while minimising mutual information between randomly paired nodes from different subsets, with the rationale that these neighbours are sensitive to random sampling of the nodes.

\begin{equation} \label{eq:contrastive}
\begin{aligned}
    \mathcal{L}_{\text{C}_{o}}(z^{\prime}_{o}, z^{\prime\prime}_{o}) = \frac{1}{k_{s}} (
    \sum_{i \in (s_1 \cup s_2)} & \mathbb{E}_{G^{\prime}_{o},G^{\prime\prime}_{o}}\left[\log{C(i,\prime,i,\prime\prime)}\right] + \\
    \sum_{
    \substack{(i,j),~i \in (s_1 \setminus s_2),\\~~j \in (s_2 \setminus s_1)}} & \mathbb{E}_{G^{\prime}_{o},G^{\prime\prime}_{o}}\left[\log{C(i,\prime,j,\prime\prime)}\right] ).
\end{aligned}
\end{equation}

Lastly, we use a reconstruction loss, with two examples provided in Equations~\ref{eq:rec-mse} and~\ref{eq:rec-bce}, by concatenating ($|$) the input and reconstruction matrices for each of the subsets $s_1$ and $s_2$. This enables our model to learn how to rebuild information from the incomplete graphs $G^{\prime}$ and $G^{\prime\prime}$. We then compute a weighted sum of the reconstruction losses and the contrastive losses, with the reconstruction losses weighted by $\alpha$ and the contrastive losses by $\beta$. This formulation yields the total loss $\mathcal{L}$, as defined in Equation~\ref{eq:total-loss}.

\begin{equation}\label{eq:total-loss}
    \mathcal{L} = \alpha\left(\sum_{o \in O} \mathcal{L}_{\operatorname{rec}_{o}}(x^{\prime}_{o}\|x^{\prime\prime}_{o},\hat{x}^{\prime}_{o}\|\hat{x}^{\prime\prime}_{o})\right) + \beta\left(\sum_{o \in O} \mathcal{L}_{\operatorname{C_{o}}}(z^{\prime}_{o},z^{\prime\prime}_{o})\right).
\end{equation}

Overall, this framework provides a scalable and efficient approach for learning joint representations from multi-omics data, even in scenarios with large sample sizes or incomplete data.

\subsection{Complexity Analysis for Subset Contrastive Learning}

In this section, we provide a time and space complexity analysis for Subset Contrastive Learning, demonstrating that our method achieves improvements over full-batch network training.  Most modern graph-based deep learning frameworks operate using edge-based operations for graph neural networks. In our analysis, we will consider a graph $G=(V,E)$, where $n=|V|$, $m=|E|$ and where $d$ is the number of features in our network.

\begin{theorem}
    For a sampling rate $k_{s} \leq \frac{n}{2}$ and linear or superlinear operations, the memory required to store the values is at most a constant factor greater than that required for the full graph with $n$ nodes.
\end{theorem}

\begin{proof}
    Given $k_{s} = \frac{n}{2}$, we have two graphs, $G^{\prime} = (V^{\prime},E^{\prime})$ and $G^{\prime\prime} = (V^{\prime\prime},E^{\prime\prime})$, each of size $k_{s}$. Since $G^{\prime}$ and $G^{\prime\prime}$ are equally sized, we can assume, without loss of generality, that all operations take double the time and space required by $G^{\prime}$.
    By definition, we have $|V^{\prime}|=k_{s}=\frac{n}{2}$ nodes in our subsampled graphs. Using k-nearest neighbour (KNN) graphs with $k_{k}$ edges per node, the number of edges in the subsampled graphs is $|E^{\prime}|\approx k_{s}k_{k}=\frac{k_{k}n}{2}$.
    Therefore, for any operation $f$ takes linear time/space w.r.t. $n$ ($O(f)=O(n+c)$) or $m$ ($O(f)=O(m+c)$), the time required for both subgraphs is:
    \begin{equation}
        O(f^{\prime}) = O(2\times (\frac{n}{2}+c)) = O(n+2\times c).
    \end{equation}
    equivalently for operations depending on the number of edges. Therefore $O(f)$ and $O(f^{\prime})$ are equal up to a constant factor.
    For operations with superlinear complexity ($k_{p}>1$), the time required becomes:
    \begin{equation}
        O(f^{\prime}) = O(2\times (({\frac{n}{2}})^{k_{p}}+c)) = O( {n}^{k_{p}} \times {2}^{1-k_{p}} + 2\times c)).
    \end{equation}
    Since $k_{p}>1$, $1-k_{p}<0$, meaning that the multiplicative factor ${2}^{1-k_{p}}$ reduces overall growth, while the constant factor increases by at most 2. A similar argument can be followed for other cases.
\end{proof}

Many constrastive learning methods require either a linear number of comparisons (e.g. DGI \cite{velickovic_deep_2019}) or a quadratic number of comparisons, as in cases where each node is compared with every other node (e.g. \cite{wang_contrastive_2024,xie_subtype-mgtp_2024,zhang_sczag_2024}). Our analysis demonstrates that, under a sampling size smaller than $\frac{n}{2}$, our method requires less memory for these linear-to-superlinear contrastive learning strategies. This reduction in memory usage, although only done up to a constant depending on $\frac{k_{s}}{n}$, is particularly crucial for datasets with large sample sizes, where the existing algorithms become computationally infeasible. For instance, we observed that the algorithm developed by \cite{wang_contrastive_2024} exhausted our memory ($\geq 64GB$) resources when applied to the dataset from \cite{kotliarov_broad_2020}.

\section{Results}

Table~\ref{tab:results-kotliarov} highlights the strong performance of our model, SCONE, in comparison to the baseline, TotalVI, on the Kotliarov2020 dataset \cite{kotliarov_broad_2020}. SCONE achieved the highest average Adjusted Mutual Information (AMI) and Adjusted Rand Index (ARI) values when leveraging both RNA and CITE-seq omics layers with the Leiden clustering algorithm. Although differences were not statistically significant compared to TotalVI under the same conditions, SCONE consistently demonstrated superior average performance within each clustering algorithm. Notably, SCONE RNA+CITE excelled in nearly all scenarios, with the exception of AMI under the Louvain algorithm, where TotalVI using only CITE-seq slightly outperformed.

Interestingly, TotalVI exhibited higher average AMI and ARI when using only CITE-seq data, compared to its performance with RNA or both omics layers combined. In contrast, SCONE consistently achieved the best performance when integrating both RNA and CITE-seq, followed by using only CITE-seq, and lastly, RNA alone.

\begin{table*}
    \centering
    \begin{tabular}{ccrrrr}
         \toprule
        \multirow{2}{*}{Model} & \multirow{2}{*}{Views} & \multicolumn{2}{c}{Louvain} & \multicolumn{2}{c}{Leiden} \\
        & & AMI & ARI & AMI & ARI \\
        \midrule
        \multicolumn{6}{c}{KOTLIAROV2020 \cite{kotliarov_broad_2020}} \\
        \midrule
        \multirow{3}{*}{TotalVI \cite{gayoso_joint_2021}}
        & RNA & $.692\pm.015$ & $.561\pm.029$ & $.703\pm.015$ & $.570\pm.030$ \\
        & CITE & $\mathbf{.781\pm.007}$ & $\mathit{.630\pm.022}$ & $^{\dagger}\mathit{.795\pm.017}$ & $.651\pm.045$ \\
        & RNA+CITE & $^{\dagger}\mathit{.746\pm.160}$ & $^{\dagger}\mathit{.618\pm.208}$ & $^{\dagger}\mathit{.756\pm.163}$ & $^{\dagger}\mathit{.639\pm.214}$ \\
        \multirow{3}{*}{SCONE (ours)}
        & RNA & $.514\pm.033$ & $.361\pm.048$ & $.567\pm.040$ & $.458\pm.052$ \\
        & CITE & $.754\pm.018$ & $.480\pm.056$ & $.777\pm.008$ & $.539\pm.023$ \\
        & RNA+CITE & $.765\pm.017$ & $\mathbf{.654\pm.038}$ & $\mathbf{^{*}.800\pm.014}$ & $^{*}\mathbf{.698\pm.015}$ \\
         \bottomrule
    \end{tabular}
    \caption{Adjusted Mutual Information (AMI) and Adjusted Rand Index (ARI) for each tested model in the Kotliarov2020 dataset \cite{kotliarov_broad_2020}. Values are reported as $\text{mean}\pm\text{standard deviation}$. The highest value for each metric and clustering algorithm is highlighted in \textbf{bold}. \textit{Italicised} values indicate results that were not statistically significantly different ($<0.05$) from the highest value on a two-tailed t-test, which could either be due to a tie or lack of samples. Symbols $^{*}$ and $^{\dagger}$ are used to denote the best value and non-significantly different values, respectively, across each metric.}
    \label{tab:results-kotliarov}
\end{table*}

Table~\ref{tab:results-tcga} showcases the strong performance of SCONE and SubtypeMGTP on the TCGA dataset. SCONE demonstrated its versatility and effectiveness, achieving the highest overall mean $-\log_{10}(p)$ values across the datasets. Notably, SCONE's multi-omics model achieved the highest average $-\log_{10}(p)$ values in four of the seven datasets, including a statistically significant improvement over SubtypeMGTP in the SKCM dataset ($p<0.05$). 

SCONE's single-omics models also showcased remarkable performance, outperforming both its multi-omics counterpart and SubtypeMGTP in five datasets based on $-\log_{10}(p)$ values. In BRCA, KIRC, and LUAD, the RNA-only SCONE model achieved significantly better results than all other models ($p<0.05$). For SKCM, SCONE RNA-only outperformed SubtypeMGTP and was comparable to its multi-omics version. In UCEC, SCONE using only the methylation layer performed on par with both its multi-omics model and SubtypeMGTP, highlighting its robustness even in single-omics scenarios.

Although SubtypeMGTP achieved the highest average $-\log_{10}(p)$ in BLCA, the differences were not statistically significant when compared to SCONE's multi-omics or single-omics models using RNA or miRNA. In STAD, SubtypeMGTP significantly outperformed SCONE's single-omics models ($p<0.05$) but did not achieve a significant improvement over SCONE's multi-omics model, further emphasising the competitive nature of SCONE's performance.

\begin{table*}
    \centering
    \scriptsize
    \begin{tabular}{lrrrrrrr}
    \toprule
    model & BLCA & BRCA & KIRC & LUAD & SKCM & STAD & UCEC \\
    \midrule
    SCONE & $^{\dagger}\mathit{1.92}\pm0.30 (2.26)$ & $\mathbf{1.18}\pm0.38 (1.59)$ & $\mathbf{6.12}\pm1.60 (9.84)$ & $\mathit{0.77}\pm0.62 (2.14)$ & $^{\dagger}\mathbf{4.69}\pm0.78 (5.75)$ & $^{\dagger}\mathit{2.40}\pm0.97 (3.80)$ & $^{\dagger}\mathbf{5.92}\pm0.65 (6.98)$ \\
    SubtypeMGTP & $^{*}\mathbf{2.04}\pm0.57 (2.71)$ & $\mathit{0.94}\pm0.44 (1.64)$ & $\mathit{4.76}\pm2.95 (9.12)$ & $\mathbf{0.98}\pm0.34 (1.62)$ & $2.59\pm0.72 (4.06)$ & $^{*}\mathbf{2.75}\pm1.20 (4.60)$ & $^{\dagger}\mathit{5.14}\pm1.52 (8.28)$ \\
    \midrule
    SCONE (\texttt{rna}) & $^{\dagger}1.66\pm0.40$ (2.28) & $^{*}1.81\pm0.45$ (2.51) & $^{*}10.16\pm2.25$ (13.97) & $^{*}2.95\pm1.23$ (4.53) & $^{*}4.85\pm1.18$ (6.49) & $0.61\pm0.45$ (1.44) & $5.22\pm0.99$ (6.59) \\
    SCONE (\texttt{prot}) & $1.17\pm0.61$ (2.54) & $1.14\pm0.64$ (2.04) & $7.46\pm1.41$ (9.31) & $0.19\pm0.18$ (0.51) & $2.20\pm0.52$ (3.02) & $1.65\pm0.33$ (2.28) & $4.33\pm0.85$ (5.22) \\
    SCONE (\texttt{miRNA}) & $^{\dagger}1.84\pm0.51$ (2.54) & $0.52\pm0.30$ (0.98) & $4.43\pm2.47$ (10.12) & $0.64\pm0.42$ (1.35) & $3.32\pm0.55$ (3.97) & $0.75\pm0.45$ (1.35) & $2.06\pm0.72$ (3.35) \\
    SCONE (\texttt{meth}) & $0.66\pm0.20$ (0.86) & $0.70\pm0.25$ (0.94) & $3.83\pm1.32$ (5.64) & $1.39\pm1.12$ (3.19) & $0.41\pm0.29$ (0.84) & $0.09\pm0.09$ (0.27) & $^{*}6.47\pm1.14$ (7.98) \\
    SCONE (\texttt{CN}) & $0.08\pm0.07$ (0.20) & $0.15\pm0.07$ (0.26) & $2.83\pm0.92$ (4.22) & $0.56\pm0.44$ (1.41) & $0.46\pm0.20$ (0.68) & $0.45\pm0.39$ (1.11) & $1.08\pm1.47$ (4.64) \\
    \bottomrule
    \end{tabular}
    \caption{Survival logrank test $-\log_{10}(p)$ values for the TCGA dataset \cite{xie_subtype-mgtp_2024} using $k$ clusters with a predefined $k$ \cite{yang_subtype-gan_2021,xie_subtype-mgtp_2024}. Values are presented as $\text{mean}\pm\text{standard deviation} (\text{max})$. The highest value for each dataset in the multi-omics setting is highlighted in \textbf{bold}. \textit{Italicised} values indicate results that were not statistically significantly different ($<0.05$) from the highest value on a two-tailed t-test.
    SCONE's results are presented both for the multi-omics setting, where use only the RNA, Protein and miRNA layers, and single-omics settings, while we follow \cite{xie_subtype-mgtp_2024}'s description for SubtypeMGTP. 
    Symbols $^{*}$ and $^{\dagger}$ are used to denote the best value and non-significantly different values, respectively, when considering all single-omics ablation and multi-omics results.
    }
    \label{tab:results-tcga}
\end{table*}

\section{Discussion}

SCONE demonstrated robust performance across multiple datasets, showcasing its ability to effectively integrate multi-omics data and preserve biological information. On the Kotliarov2020 dataset, SCONE exhibited a narrower performance distribution compared to TotalVI's multi-omics model, indicating a more consistent clustering performance. Notably, TotalVI achieved its best performance using only the CITE-seq data, which was also used to define the dataset's "ground truth" labels \cite{kotliarov_broad_2020}. This result suggests that TotalVI's integration strategy may act antagonistically on this dataset, prioritising one omics layer over the other rather than fostering a synergistic interaction. In contrast, SCONE's multi-omics integration strategy enabled a synergistic interaction between omics layers, effectively leveraging complementary information to outperform models utilising only a single omics layer. For example, SCONE’s integration of RNA and CITE-seq data compensated for potential noise in individual layers, allowing it to achieve superior performance compared to models relying solely on the CITE-seq layer.

\begin{figure*}
    \centering
    \ifx\usevectorimages\undefined
    \includegraphics[width=0.9\linewidth]{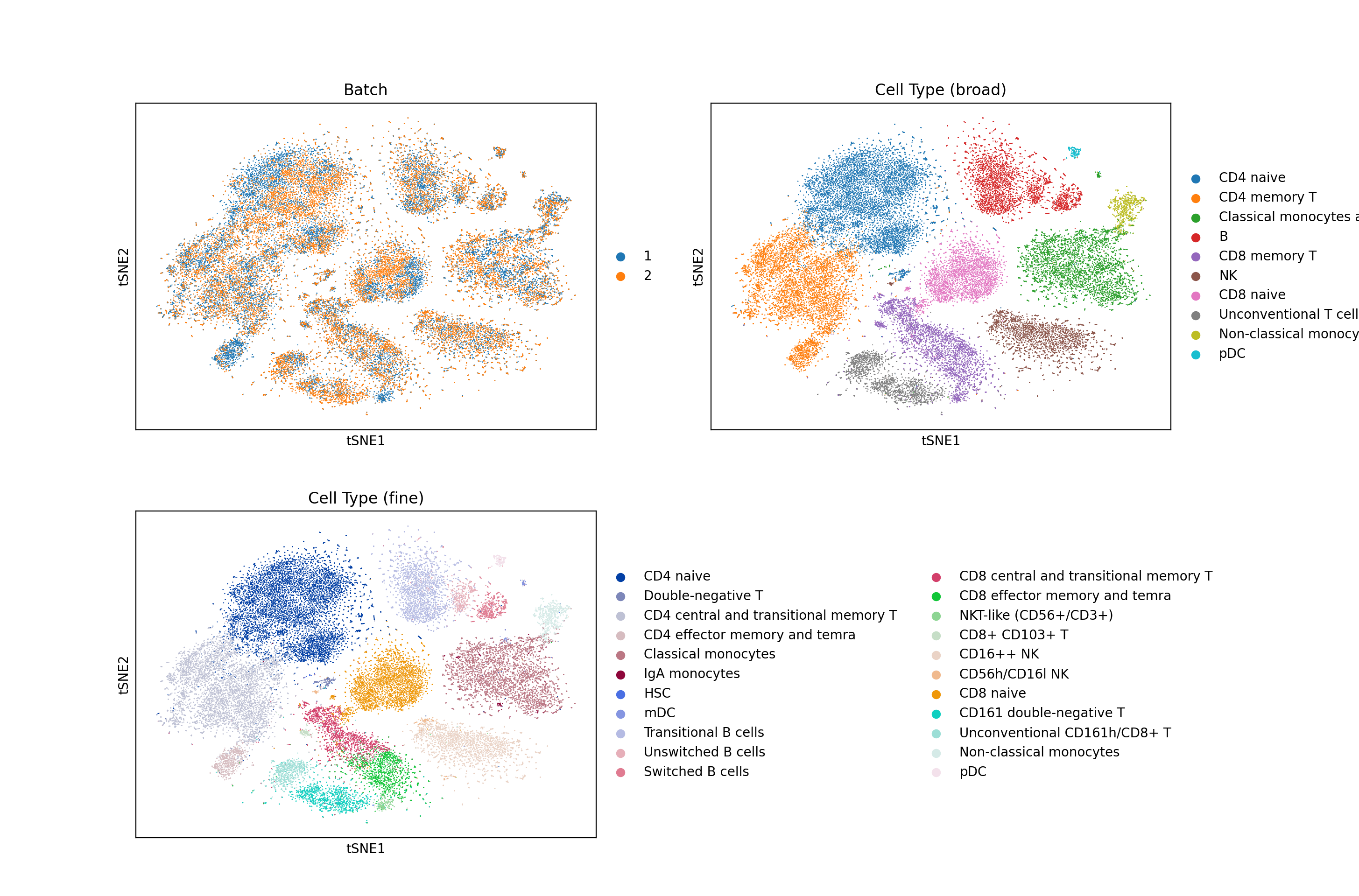}
    \else
    \includegraphics[width=0.9\linewidth]{Fig/tsne-kotliarov2020-subsetcontrastive-celltypes.pdf}
    \fi \hspace{0.3\linewidth}
    \caption{A t-SNE \cite{maaten_visualizing_2008} visualisation of one of our multi-omics model's learned latent representation of the \cite{kotliarov_broad_2020} dataset. On the top left, we overlay the two batches defined by \cite{lotfollahi_multigrate_2021} which correspond to low- and high- vaccine responders. On the top right we overlay the broad cell type ground-truth defined by \cite{kotliarov_broad_2020} through analysis of the surface proteins in their dataset, which we used as our ground-truth for comparisons. In the bottom we provide the fine cell-type annotation provided by \cite{kotliarov_broad_2020} to showcase how our model separates these subtypes within the broader cell-types.  }
    \label{fig:tsne-kotliarov-subtypes}
\end{figure*}

Importantly, SCONE's ability to capture meaningful biological insights is evident in its learned latent manifold. As illustrated in Figure~\ref{fig:tsne-kotliarov-subtypes}, out RNA+CITE SCONE successfully clusters similar cell types, demonstrating clear separation for both the broad and fine cell-type definitions provided by \cite{kotliarov_broad_2020}. In fact, although our single-omics models underperformed relative to TotalVI's single-omics models, our CITE-seq-only model was not the best-performing model within our framework. We hypothesise that the subset sampling procedure employed in our method may have introduced noise, potentially contributing to the lower average performance metrics observed in our single-omics models compared to TotalVI's equivalent models. However, despite this limitation, our multi-omics integration strategy enabled a synergistic interaction between the omics layers, effectively compensating for the noise and leveraging complementary information.

This interaction allowed our model to complement CITE-seq data using information from the RNA omics layer, resulting in superior performance even compared to models utilising only the omics layer that was originally used to derive the ground truth labels for this dataset. This observation is further supported by the visualisation of surface-protein cell-type markers \cite{kotliarov_broad_2020}, as shown in Figure~\ref{fig:tsne-kotliarov-markers}, that demonstrate a clear biological signal for each cell type cluster. Notably, even the hematopoietic stem cell (HSC) markers, which correspond to a cluster comprising only 57 cells (0.11\% of the dataset), remain tightly clustered in the centre of our t-SNE visualisation \cite{maaten_visualizing_2008}, albeit in close proximity to other clusters. This level of granularity underscores the ability of our model to preserve biological information, even for rare cell populations, in its learned manifold.

\begin{figure*}
    \centering
    \includegraphics[width=0.9\linewidth]{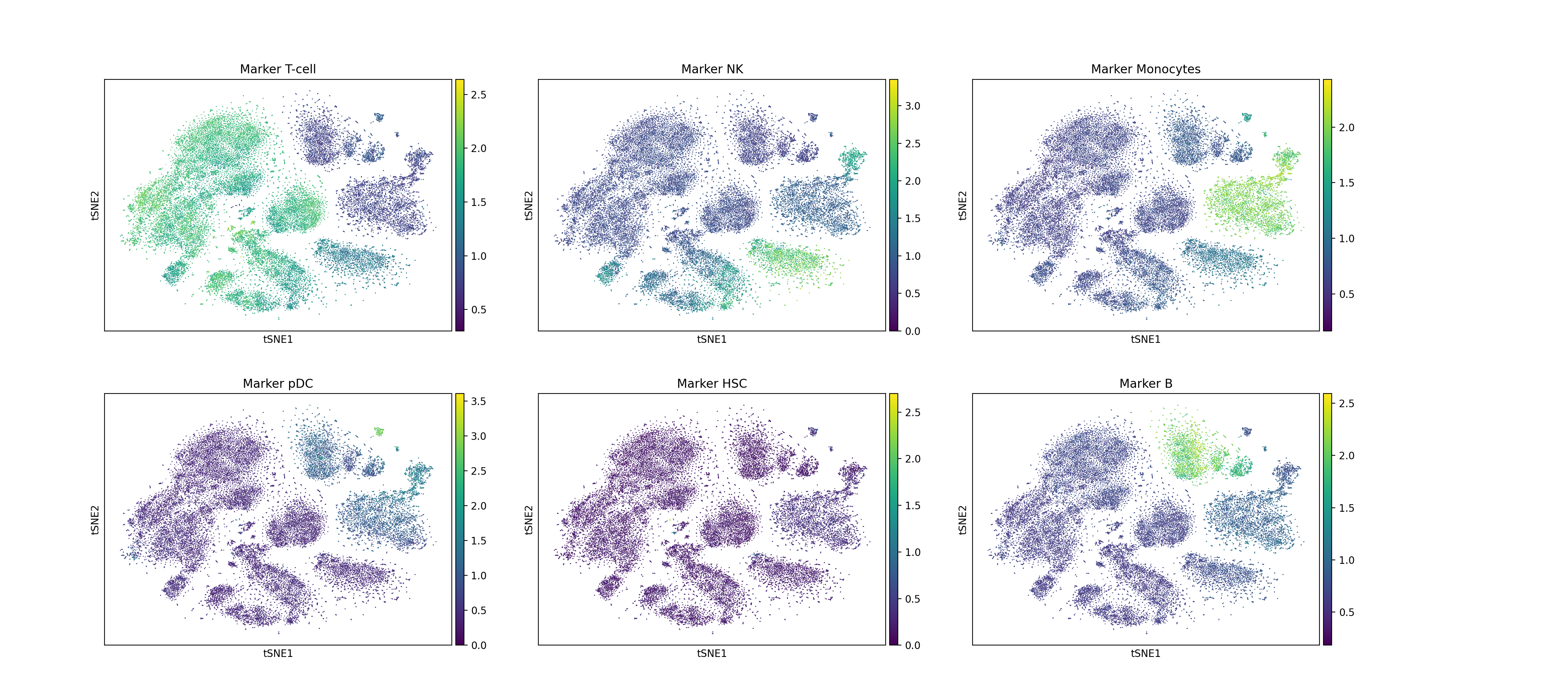} \\
    \caption{A t-SNE \cite{maaten_visualizing_2008} visualisation of one of our multi-omics model's learned latent representation of the \cite{kotliarov_broad_2020} dataset. We overlay the average of the $\log_{10}$ surface-protein expression for the marker proteins of each cell type defined by \cite{kotliarov_broad_2020}. From left-to-right, top-to-bottom we have the average surface-protein expression for T-cells, Monocytes, Natural Killer (NK) cells, plasmacytoid Dendritic Cells (pDC), Hematopoietic stem cells (HSC), and B-cells.
    A more detailed per-marker-protein view can be seen in Supplementary Figures~\ref{sup:fig:marker-t}, \ref{sup:fig:marker-monocyte}, \ref{sup:fig:marker-nk}, \ref{sup:fig:marker-pdc}, \ref{sup:fig:marker-hsc}, and~\ref{sup:fig:marker-b}, which can further show how our integrated topology still retains the crucial surface protein information.
     }
    \label{fig:tsne-kotliarov-markers}
\end{figure*}

In the context of TCGA datasets, SCONE’s single-omics RNA models consistently yielded higher $-\log_{10}(p)$ values compared to protein models, except in the STAD dataset, where protein markers were particularly informative for patient survival. The rationale behind SubtypeMGTP's multi-omics-to-protein translation is that the protein omics layer is the most informative \cite{xie_subtype-mgtp_2024}. Our experiments have shown that this might only be the case specifically in the STAD dataset, where the protein-only SCONE model layer achieved the highest $-\log_{10}(p)$ value. In this dataset, SubtypeMGTP also significantly outperformed all our single-omics models, suggesting that protein markers might be particularly informative for patient survival in the STAD context, leading to SubtypeMGTP’s superior overall performance. However, in all other datasets, our single-omics RNA models consistently yielded higher $-\log_{10}(p)$ values compared to the protein models.

Interestingly, while the multi-omics integration in SCONE did not always result in a purely synergistic interaction, it often provided an advantage over analysing omics layers individually and averaging the results. Synergistic behaviour was observed in only three datasets: BLCA, STAD, and UCEC. In three other datasets, the average performance of our single-omics models (RNA, Protein, and miRNA) was within one standard deviation of the average performance of our multi-omics model. Specifically, in the BRCA dataset the single-omics ensemble was $0.05\sigma$ below the multi-omics model. In the KIRC and LUAD datasets, the single-omics ensemble outperformed the multi-omics model by $0.77\sigma$ and $0.79\sigma$, respectively. Interestingly, in the SKCM dataset, although our RNA-seq-only model exhibited the highest average $-\log_{10}(p)$ value, the single-omics ensemble performed $1.58\sigma$ below the multi-omics model. This suggests that while the multi-omics integration did not always act purely synergistically, it often provided an advantage over analysing the omics layers individually.

\section{Conclusion}

In this study, we introduced the Subset-Contrastive multi-Omics Network Embedding (SCONE) model, which demonstrated strong performance across a range of single- and multi-omics datasets. By leveraging a subgraph-based contrastive learning approach, SCONE achieved synergistic integration of multiple omics layers, enabling it to effectively capture and utilise complementary information. Despite the challenges posed by the subset sampling procedure, SCONE maintained competitive performance, frequently outperforming both network-based and non-network-based baselines.

Our findings highlight the potential of subgraph-based contrastive methods for analysing large-scale molecular datasets. SCONE's ability to integrate diverse omics layers while maintaining computational efficiency positions it as a valuable tool for multi-omics analyses. Furthermore, our results underscore the importance of exploring subset-based techniques, particularly in the context of emerging spatial transcriptomics data, where constraints on data availability and resolution present unique challenges.

We anticipate that SCONE will motivate further research into subset-based contrastive methods and their application to multi-omics and spatial transcriptomics datasets. Future work may focus on refining integration strategies to enhance synergy between omics layers, as well as extending SCONE's framework to accommodate the specific demands of spatially resolved data.

\section*{Acknowledgements} MW and ST acknowledge funding from King's College London and the A*STAR Research Attachment Programme (ARAP) to PHCA. MW acknowledges funding by the AI, Analytics and Informatics (AI3) Horizontal Technology Programme Office (HTPO) seed grant (grant no: C211118015) from A*STAR, Singapore. ST acknowledges funding from the British Skin Foundation (006/R/22) and the UK Royal Society (IES\textbackslash R2\textbackslash 222084). We'd like to thank Sezin Kırcalı Ata for helpful discussions. The results shown here are in whole or part based upon data generated by the TCGA Research Network: \url{https://www.cancer.gov/tcga}

\bibliographystyle{abbrv}
\bibliography{netemo}

\clearpage

\onecolumn

\setcounter{table}{0} \renewcommand\thetable{S\arabic{table}}
\setcounter{figure}{0} \renewcommand\thefigure{S\arabic{figure}}
\setcounter{section}{0} \renewcommand\thesection{S}
\setcounter{subsection}{0} \renewcommand\thesubsection{S\arabic{subsection}}

\section*{Supplementary Material} \label{sup}

\subsection{Experimental Setup} \label{sup:exp}

We selected two strong baselines for comparison: (i) TotalVI \cite{gayoso_joint_2021} and (ii) SubtypeMGTP \cite{xie_subtype-mgtp_2024}. TotalVI is a widely recognised multi-omics single-cell representation learning method, extensively used and integrated into several established and emerging libraries \cite{gayoso_python_2022,virshup_scverse_2023,ergen_scvi-hub_2024}. SubtypeMGTP, on the other hand, is a recently published multi-omics bulk representation learning method that has demonstrated significant improvements over prior techniques. For both baselines, we used the authors' publicly available code and adhered to their recommended settings. We also explored other Graph Contrastive Learning methods \cite{xie_subtype-mgtp_2024,zhang_sczag_2024,wang_contrastive_2024}. However, their high spatial complexity prevented their application to our single-cell datasets.

\subsubsection{Baseline Models}

TotalVI is a multi-modal Variational Autoencoder \cite{kingma_auto-encoding_2014} tailored for RNAseq and CITE-seq (surface protein) data. It employs a standard normal prior and reconstructs RNAseq data using a Negative Binomial (NB) distribution, while protein data is reconstructed using an NB mixture. We utilised the implementation available in the \texttt{scvi} Python package \cite{gayoso_python_2022} (version \texttt{0.14.6}), produced by the same first author as the original TotalVI paper. The latent representation was extracted using the \texttt{get\_latent\_representation()} method of the \texttt{TOTALVI} class.

SubtypeMGTP is a two-phase model. The first phase trains a multi-omics-to-protein translation module using Graph Convolutional Network (GCN) \cite{kipf_semi-supervised_2017} layers on all available protein data. Then it converts all multi-omics measurements into protein representations, which are then used as input for the second phase. In the second phase, SubtypeMGTP builds a Autoencoder model comprising GCN-based Encoder and Decoder, and then applies a Deep Subspace Contrastive Clustering module, along with a contrastive self-expression layer in the latent space. We used the original implementation available at \url{github.com/kybinn/Subtype-MGTP}, with an archived version hosted at \url{github.com/phcavelar/Subtype-MGTP}.

\subsubsection{Proposed Model}

Our model employs mirrored encoders and decoders, with each omics layer's encoder consisting of two GATv2 layers with $256$ and $128$ units, followed by LeakyReLU activation. The decoder mirrors this structure, with 256 units in the first layer and $d_o$ output units. Training was performed using the Adam optimiser \cite{kingma_adam_2015} with $\beta_1=0.9$, $\beta_2=0.999$, and a learning rate of $0.0001$. We used $\alpha=1$ and $\beta=10$ in Equation~\ref{eq:total-loss}. We did not optimise our model's hyperparameters for each experiment, only using different values for $k_{s}$ and number of epochs, as detailed below. For subset sampling, we ensured that at least 10\% of nodes overlapped between subsets ($|s_1 \cap s_2| \geq 0.1 \times k_s$).

Post-training, we performed clustering using either Louvain \cite{blondel_fast_2008} or Leiden \cite{traag_louvain_2019} methods. For Louvain clustering, we used the \texttt{louvain\_communities} function from the \texttt{networkx} library \cite{hagberg_exploring_2008} (version \texttt{3.3}). For Leiden clustering, we used the \texttt{find\_partition} function from the \texttt{leidenalg} library \cite{traag_louvain_2019} (version \texttt{0.9.1}). The resolution parameter was varied over 20 values in the $(0, 2]$ interval, selecting the clustering with the highest modularity. KNN graphs (k=15) were constructed using the \texttt{nn.neighbors} function from the \texttt{scanpy} library \cite{wolf_scanpy_2018} (version \texttt{1.10.2}). 

\subsubsection{Single-Cell Data}

For the single-cell dataset \cite{kotliarov_broad_2020}, we trained each model using RNA and CITE-seq data individually and combined (RNA+CITE). TotalVI was adapted to single-omics scenarios by providing a dummy matrix of zeros for the missing omics layer, with a single column containing zeros for all samples. We used subsets containing 20\% of the dataset ($k_{s} = 0.2 \times n$), corresponding approximately to the proportion of the most abundant cell type, and trained the models for 128 epochs. Both Louvain and Leiden clustering were applied, and results were reported separately.

Cluster assignments were evaluated using the Adjusted Rand Index (ARI) \cite{hubert_comparing_1985} and Adjusted Mutual Information (AMI) \cite{nguyen_information_2010}, calculated using the \texttt{adjusted\_rand\_score} and \texttt{adjusted\_mutual\_info\_score} functions from the \texttt{scikit-learn} library \cite{pedregosa_scikit-learn_2011} (version \texttt{1.4.2}).  The cluster assignments were compared with the original clustering labels provided by \cite{kotliarov_broad_2020} for ARI and AMI calculation, which were used as the performance indicator for each model. 
Cluster assignments were compared to the original labels provided by \cite{kotliarov_broad_2020}. Each model was trained eight times, and clustering was performed on each latent space, yielding eight clusterings per model/clustering method combination. We report the mean and standard deviation of ARI and AMI for each configuration.

\subsubsection{Batch Data}

For Batch data, we followed the preprocessing steps in \cite{xie_subtype-mgtp_2024}. Model performance was assessed using the logrank p-value of survival separation, calculated with the \texttt{multivariate\_logrank\_test} function from the \texttt{lifelines} library \cite{davidson-pilon_lifelines_2024} (version \texttt{0.30.0}). A pre-defined number of clusters for each cancer subtype was used, as defined by \cite{yang_subtype-gan_2021} and shown in Table~\ref{tab:dset-tcga}, for fair comparison. Unlike \cite{xie_subtype-mgtp_2024}, we trained models separately for each cancer subtype, ensuring no information flowed between datasets, and conducted multiple tests for each model. We chose to use a fixed value of $k_{s}=200$ samples per subset for each of the datasets, as this allowed us to have a consistent number of samples and $k_{s}=200$ is close to half the average dataset size (datasets had on average 215 samples).  We trained each model for 1,024 epochs, as smaller sample sizes allowed for faster training compared to the single-cell case.

As with single-cell data, we ran each model 8 times, extracted its latent space, and performed clustering on the latent space. To keep consistent with our single-cell experiments, we only used the equivalent omics layers encoding RNA, Protein, and Micro-RNA when training the model in a multi-omics scenario, also providing the performance for each omics layer when used individually. Since the Leiden clustering algorithm provides an improvement over Louvain \cite{traag_louvain_2019} and is faster to compute, we decided to only perform clustering with Leiden. Since Leiden clustering does not provide a clustering with a fixed number of clusters, we adapted our optimisation procedure and only considered clusterings that had the target number of clusters. After clustering, we calculated the logrank p-value of the clusters' survival separation. For SubtypeMGTP, we follow its second phase and  used spectral clustering with the pre-specified number of clusters for their model. We report the mean and standard deviation of the negative $\log_{10}$ p-value for fair comparison. We also include the highest values we achieved for each of the models since, given our results using the code provided by \cite{xie_subtype-mgtp_2024}, we believe that they reported their best model instead of an average value over multiple runs.

\begin{figure}[hbpt]
    \centering
    \ifx\usevectorimages\undefined
    \includegraphics[width=0.9\linewidth]{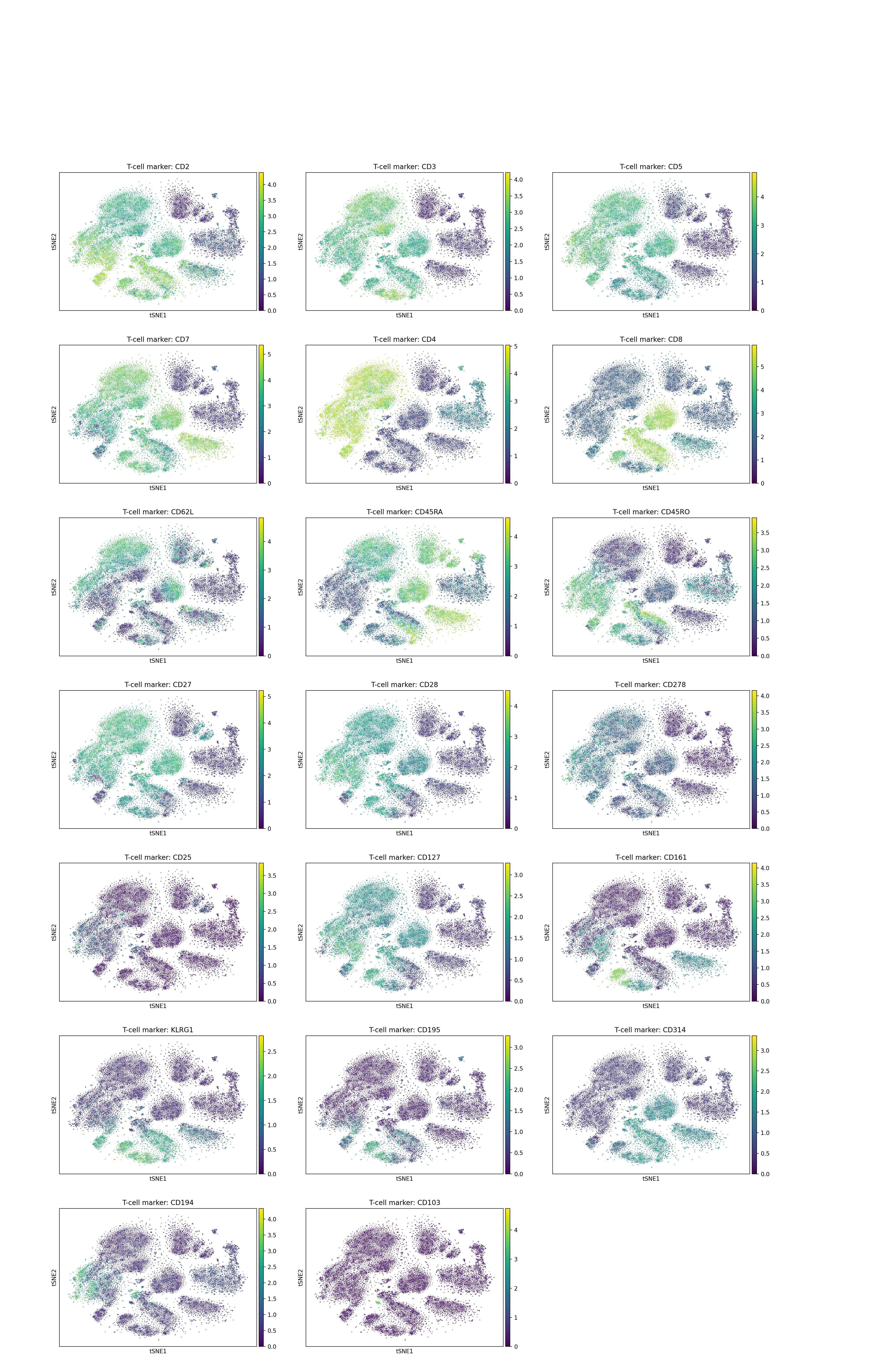}
    \else
    \includegraphics[width=0.9\linewidth]{Fig/Supp/tsne-kotliarov2020-subsetcontrastive-markers-T-cell.pdf}
    \fi
    \caption{A t-SNE \cite{maaten_visualizing_2008} visualisation of one of our multi-omics model's learned latent representation of the \cite{kotliarov_broad_2020} dataset. We overlay the $\log_{10}$ surface-protein expression for the marker proteins of T-cells defined by \cite{kotliarov_broad_2020}.  }
    \label{sup:fig:marker-t}
\end{figure}

\begin{figure}[hbpt]
    \centering
    \ifx\usevectorimages\undefined
    \includegraphics[width=0.9\linewidth]{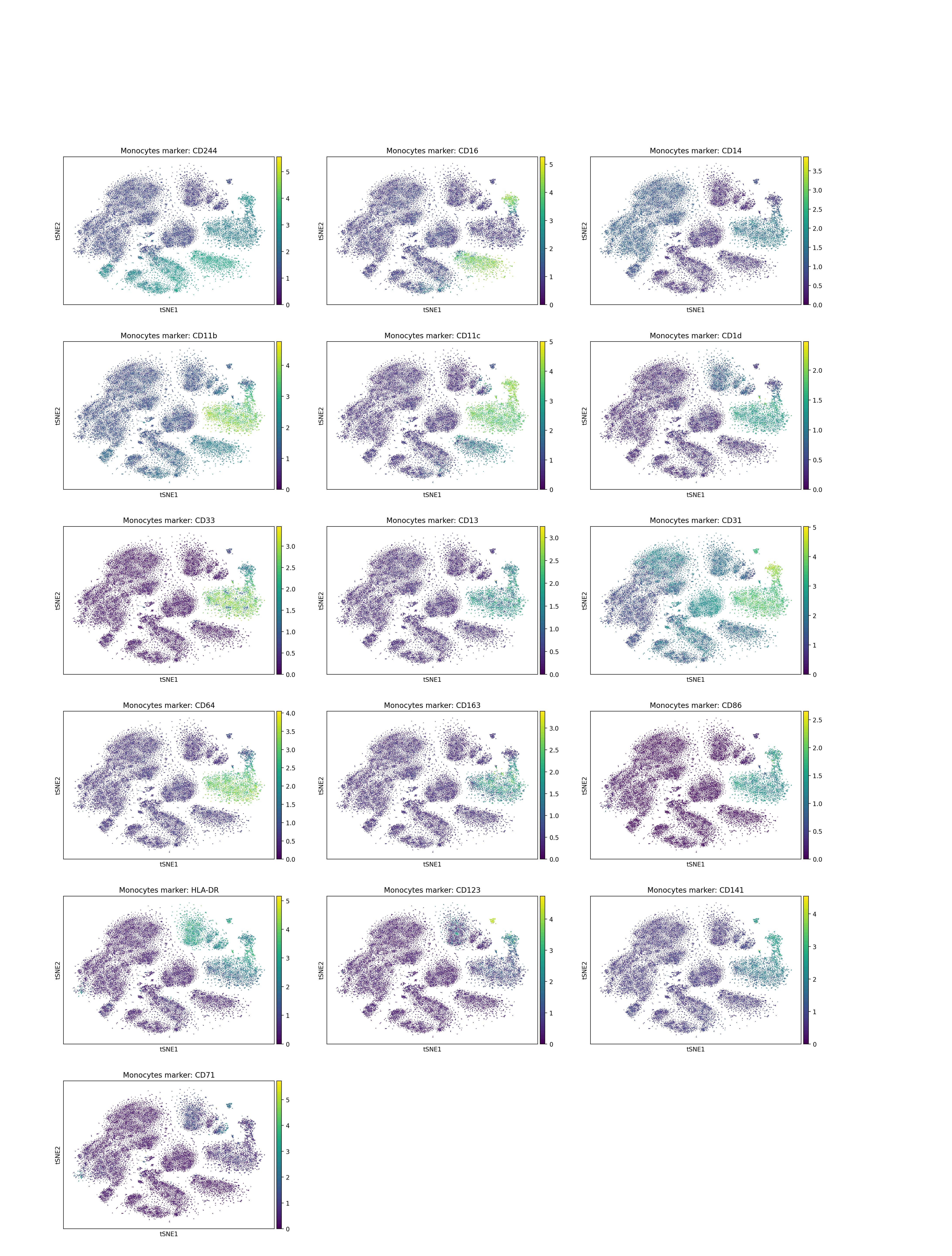}
    \else
    \includegraphics[width=0.9\linewidth]{Fig/Supp/tsne-kotliarov2020-subsetcontrastive-markers-Monocytes.pdf}
    \fi
    \caption{A t-SNE \cite{maaten_visualizing_2008} visualisation of one of our multi-omics model's learned latent representation of the \cite{kotliarov_broad_2020} dataset. We overlay the $\log_{10}$ surface-protein expression for the marker proteins of Monocytes defined by \cite{kotliarov_broad_2020}.  }
    \label{sup:fig:marker-monocyte}
\end{figure}

\begin{figure}[hbpt]
    \centering
    \ifx\usevectorimages\undefined
    \includegraphics[width=0.9\linewidth]{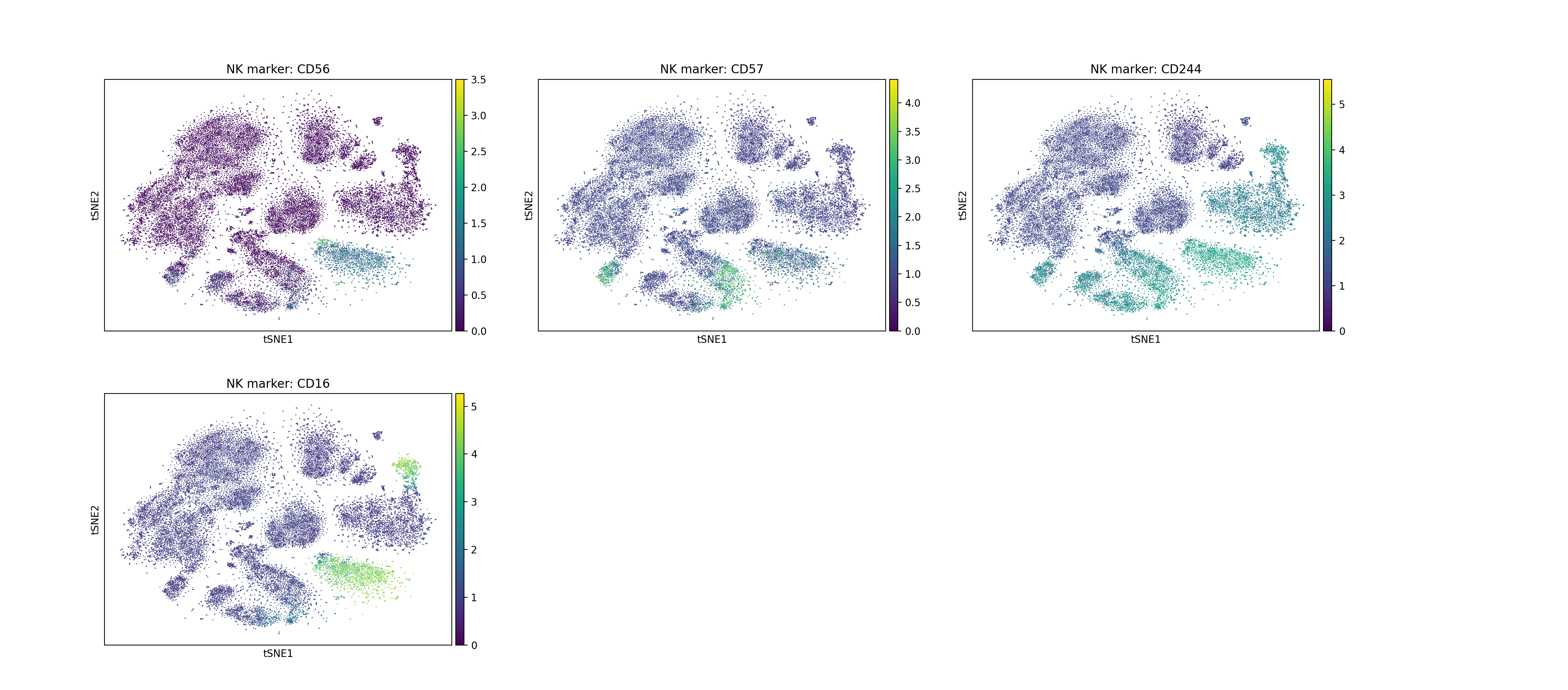}
    \else
    \includegraphics[width=0.9\linewidth]{Fig/Supp/tsne-kotliarov2020-subsetcontrastive-markers-NK.pdf}
    \fi
    \caption{A t-SNE \cite{maaten_visualizing_2008} visualisation of one of our multi-omics model's learned latent representation of the \cite{kotliarov_broad_2020} dataset. We overlay the $\log_{10}$ surface-protein expression for the marker proteins of Natural Killer (NK) cells defined by \cite{kotliarov_broad_2020}.  }
    \label{sup:fig:marker-nk}
\end{figure}

\begin{figure}[hbpt]
    \centering
    \ifx\usevectorimages\undefined
    \includegraphics[width=0.9\linewidth]{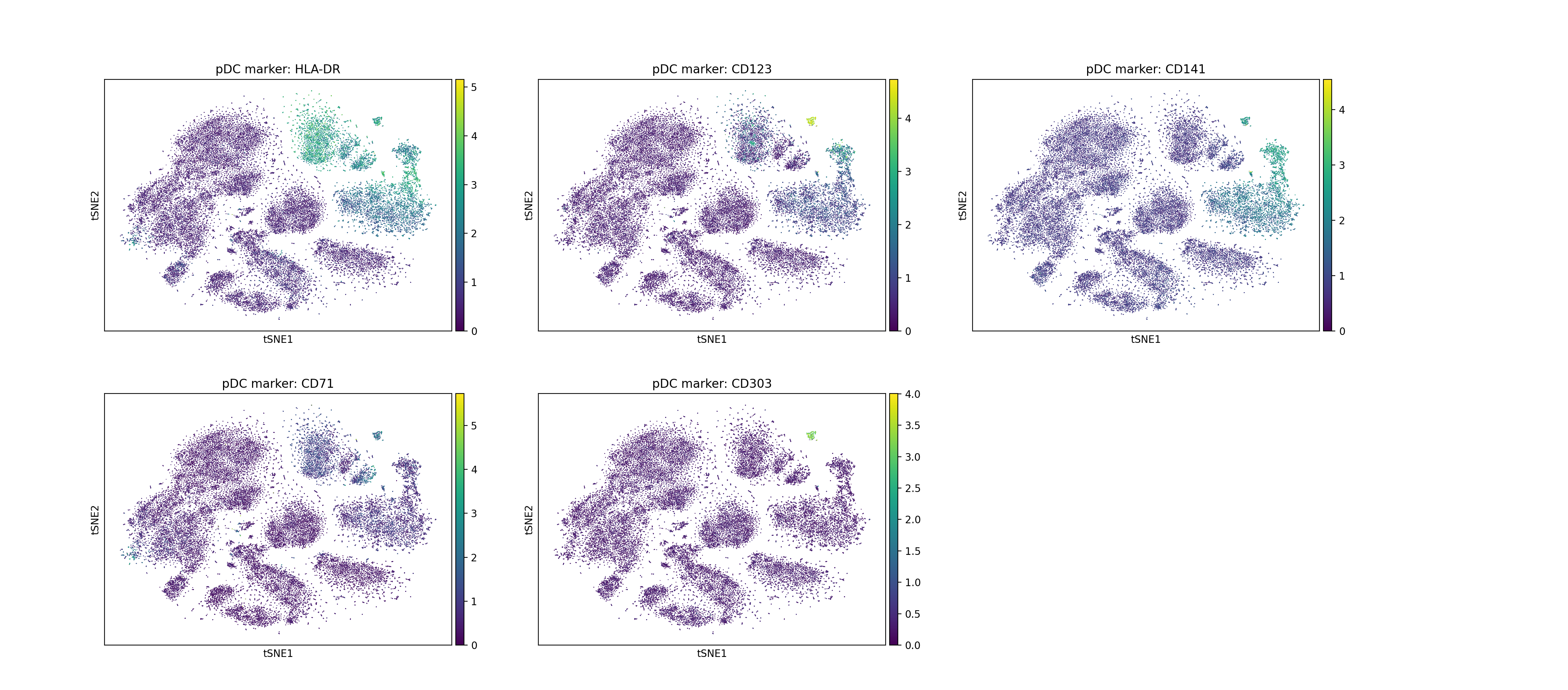}
    \else
    \includegraphics[width=0.9\linewidth]{Fig/Supp/tsne-kotliarov2020-subsetcontrastive-markers-pDC.pdf}
    \fi
    \caption{A t-SNE \cite{maaten_visualizing_2008} visualisation of one of our multi-omics model's learned latent representation of the \cite{kotliarov_broad_2020} dataset. We overlay the $\log_{10}$ surface-protein expression for the marker proteins of pDC cells defined by \cite{kotliarov_broad_2020}.  }
    \label{sup:fig:marker-pdc}
\end{figure}

\begin{figure}[hbpt]
    \centering
    \ifx\usevectorimages\undefined
    \includegraphics[width=0.6\linewidth]{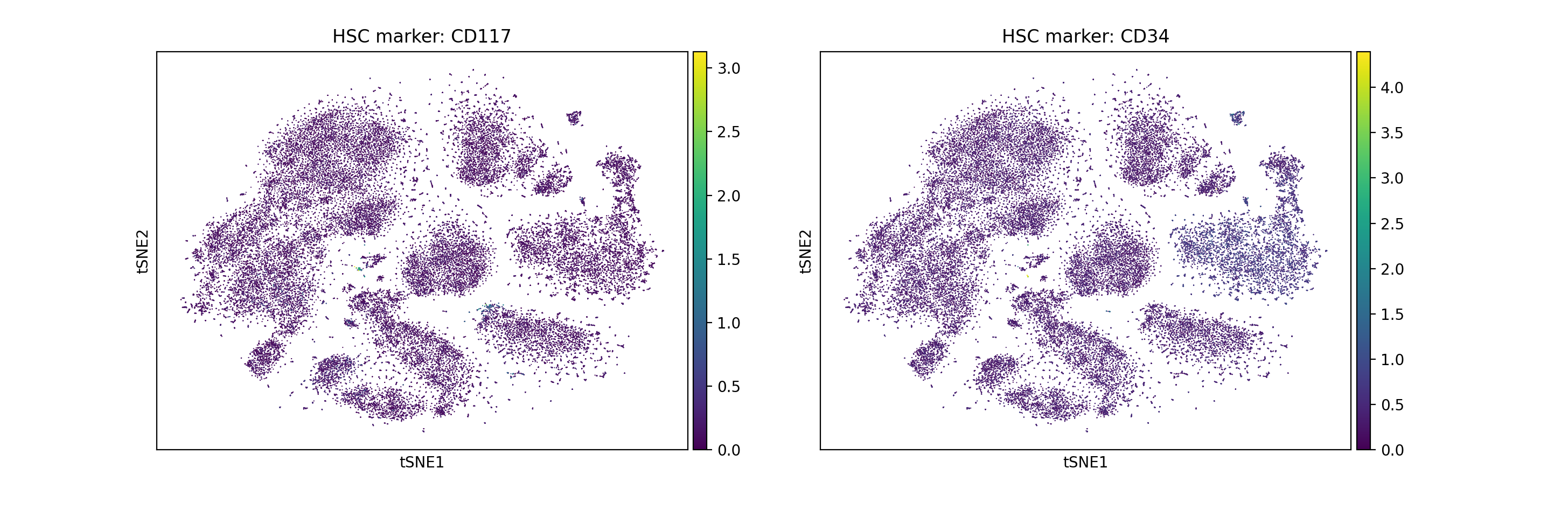}
    \else
    \includegraphics[width=0.6\linewidth]{Fig/Supp/tsne-kotliarov2020-subsetcontrastive-markers-HSC.pdf}
    \fi \hspace{0.3\linewidth}
    \caption{A t-SNE \cite{maaten_visualizing_2008} visualisation of one of our multi-omics model's learned latent representation of the \cite{kotliarov_broad_2020} dataset. We overlay the $\log_{10}$ surface-protein expression for the marker proteins of HSC cells defined by \cite{kotliarov_broad_2020}.  }
    \label{sup:fig:marker-hsc}
\end{figure}

\begin{figure}[hbpt]
    \centering
    \ifx\usevectorimages\undefined
    \includegraphics[width=0.9\linewidth]{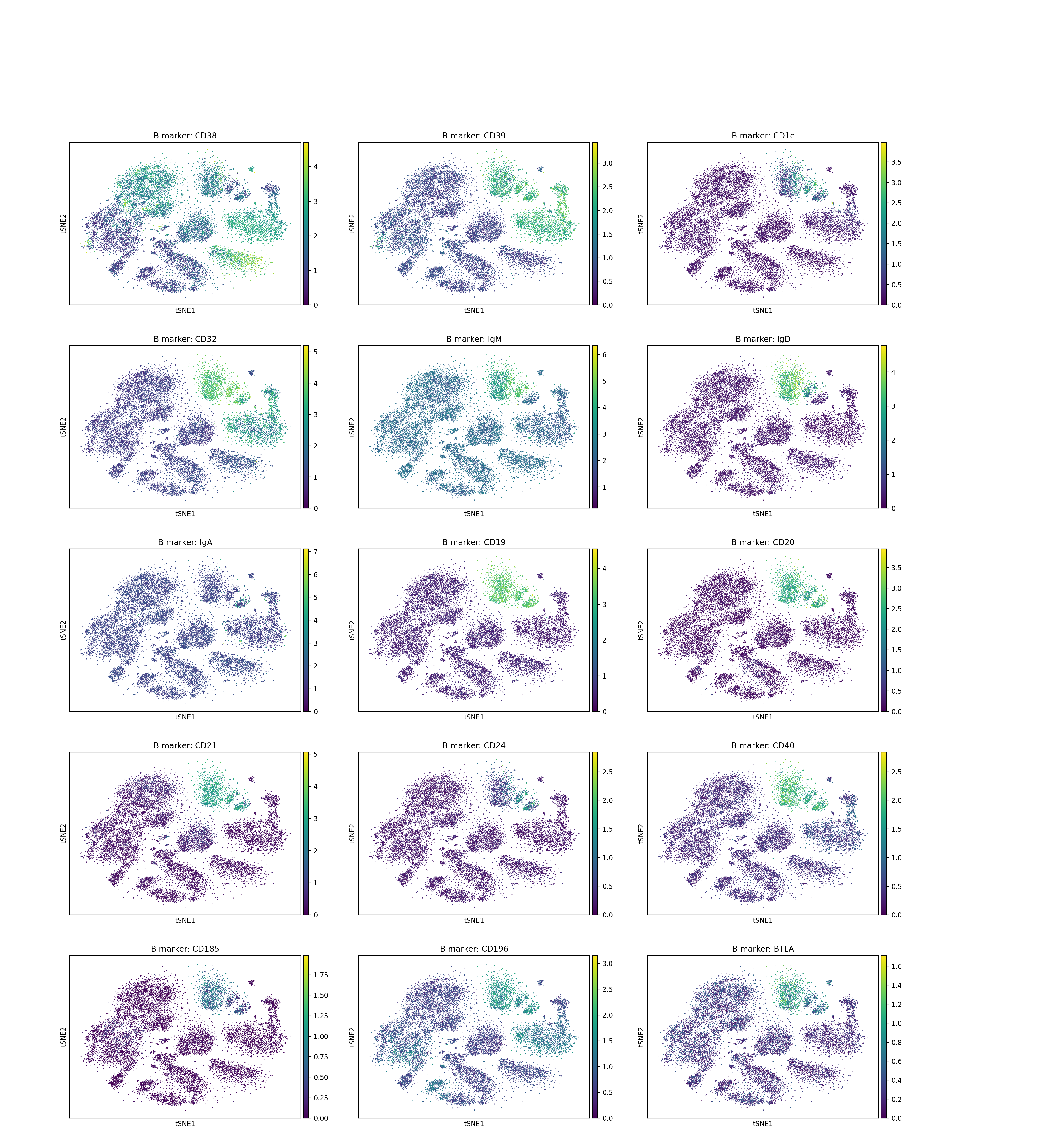}
    \else
    \includegraphics[width=0.9\linewidth]{Fig/Supp/tsne-kotliarov2020-subsetcontrastive-markers-B.pdf}
    \fi
    \caption{A t-SNE \cite{maaten_visualizing_2008} visualisation of one of our multi-omics model's learned latent representation of the \cite{kotliarov_broad_2020} dataset. We overlay the $\log_{10}$ surface-protein expression for the marker proteins of B-cells defined by \cite{kotliarov_broad_2020}.  }
    \label{sup:fig:marker-b}
\end{figure}

\end{document}